\documentclass[letterpaper]{article}
\usepackage{aaai}
\usepackage{times}
\usepackage{helvet}
\usepackage{courier}

\usepackage{mathptmx}
\usepackage{url}
\usepackage{caption}

\newcommand{\true}{{\bf true}}
\newcommand{\false}{{\bf false}}

\usepackage{verbatim}
\usepackage[all]{xy}
\usepackage{pgf, pgfnodes, pgfarrows}
\usepackage[sumlimits]{amsmath}
\usepackage{amsthm}
\usepackage{amssymb}
\usepackage{amstext}
\usepackage{amscd}

\newtheorem{lemma}{Lemma}
\newtheorem{theorem}{Theorem}

\newtheorem{definition}{Definition}

\newtheorem*{intrinsic}{Intrinsic}

\begin{document}
%
\title{Combining Probabilistic, Causal, and Normative Reasoning in CP-logic}
\author{
Sander Beckers \and Joost Vennekens\\
KU Leuven -- University of Leuven\\
\{Sander.Beckers,Joost.Vennekens\}@cs.kuleuven.be
}
\nocopyright
\maketitle
\begin{abstract}
\begin{quote}
In recent years the search for a proper formal definition of actual causation -- i.e., the relation of cause-effect as it is instantiated in specific observations, rather than general causal relations -- has taken on impressive proportions. In part this is due to the insight that this concept plays a fundamental role in many different fields, such as legal theory, engineering, medicine, ethics, etc. Because of this diversity in applications, some researchers have shifted focus from a single idealized definition towards a more pragmatic, context-based account. For instance, recent work by Halpern and Hitchcock draws on empirical research regarding people's causal judgments, to suggest a graded and context-sensitive notion of causation. Although we sympathize with many of their observations, their restriction to a merely qualitative ordering runs into trouble for more complex examples. Therefore we aim to improve on their approach, by using the formal language of CP-logic (Causal Probabilistic logic), and the framework for defining actual causation that was developed by the current authors using it. First we rephrase their ideas into our quantitative, probabilistic setting, after which we modify it to accommodate a greater class of examples. Further, we introduce a formal distinction between statistical and normative considerations.
\end{quote}
\end{abstract}

\noindent 

\section{Introduction}

In their forthcoming article {\em Graded Causation and Defaults} Halpern and Hitchcock -- HH -- quite rightly observe that not only is there a vast amount of disagreement regarding actual causation in the literature, but there is also a growing number of empirical studies which show that people's intuitions are influenced to a large degree by factors which up to now have been ignored when dealing with causation. For example, our judgments on two similarly modelled cases may differ depending on whether it takes place in a moral context or a purely mechanical one, or on whether what we take to be the default setting, or on whether we take something to be a background condition or not, etc. \cite{knobe08,moore09,hitchcockknobe09}. This has led HH to develop a flexible framework that allows room for incorporating different judgments on actual causation. More specifically, in their view the difference between cases that are modelled using similar structural models depends on which worlds we take to be more normal than others in the different contexts. Therefore their solution is to extend structural models with a normality ranking on worlds, and use it to adapt and order our judgments of actual causation in a manner suited for the particular context.

We sympathize with many of their observations, and we agree that normality considerations do influence our causal judgments. However, we find their representation of normality lacking for two reasons. First, although they emphasize the importance of distinguishing between statistical and normative normality, they use a single ranking for both. Second, they refrain from using probabilities to represent statistical normality, and instead work with a partial preorder over worlds. In this paper we offer an alternative, in which statistical normality is represented in the usual way, i.e., by means of probabilities. As we will show, such a quantitative representation of statistical normality avoids a number of problems that HH's ordinal representation runs into. To cope with normative normality, we introduce a separate notion of norms. As a technical tool for our analysis, we will use a normative extension of a general framework for the study of actual causation \cite{hallpaper} that was defined in the language of CP-logic (Causal Probabilistic logic) \cite{vennekens09}. An important property of our technical approach is that we implement the concept of normality by means of a syntactic transformation, similar to how previous work has defined the concept of interventions in CP-logic \cite{vennekens:jelia}. This allows us to improve on the HH approach, by properly accommodating the counterfactual nature of causation. The result will be a more generally applicable and yet simpler approach.

In Section \ref{sec:cp} we shortly present the formal language of CP-logic and use it to formulate a definition of actual causation. The following section presents the extension to actual causation by HH. We translate their work into the CP-logic framework in Section \ref{sec:hhcp}. Section \ref{sec:cd} contains a first improvement to this translation, followed by some examples and our final extension to actual causation in Section \ref{sec:norms}.

We will use the following story from \cite{knobe08} as our running example, as it illustrates the influence normative considerations can have on our causal attributions:
\begin{quote}
The receptionist in the philosophy department keeps her desk stocked with pens. The administrative assistants are allowed to take pens, but faculty members are supposed to buy their own. The administrative assistants typically do take the pens. Unfortunately, so do the faculty members. The receptionist repeatedly e-mailed them reminders that only administrators are allowed to take the pens. On Monday morning, one of the administrative assistants encounters professor Smith walking past the receptionist’s desk. Both take pens. Later, that day, the receptionist needs to take an important message...but she has a problem. There are no pens left on her desk.
\end{quote}

\section{Preliminaries: CP-logic}\label{sec:cp}

We give a short, informal introduction to CP-logic. A detailed description can be found in \cite{vennekens09,vennekens:jelia}. The basic syntactical unit of CP-logic is a CP-law, which takes the general form of $Head \leftarrow Body$. The body can in general consist of any first-order logic formula. However, in this paper, we restrict our attention to grounded formulas in CNF. The head contains a disjunction of atoms annotated with probabilities, representing the possible effects of this law. When the probabilities in a head do not add up to one, we implicitly assume an {\em empty} disjunct, annotated with the remaining probability.

Each CP-law models a specific {\em causal mechanism}. Informally, we take this to mean that, if the $Body$ of the law is satisfied, then at some point it will be applied, meaning one of the disjuncts in the $Head$ is chosen, each with their respective probabilities. If a disjunct is chosen containing an atom that is not yet $\true$, then this law causes it to become $\true$; otherwise, the law has no effect. A finite set of such CP-laws forms a CP-theory, and represents the causal structure of the domain at hand. The domain unfolds by laws being applied one after another, where multiple orders are often possible, and each law is applied at most once. To illustrate this, we describe the relevant events on the Monday morning from our running example. The domain consists of the variables $Prof$ and $Assistant$, which stand for the professor respectively the assistant taking a pen, and $NoPens$, which is true when there are no pens left.
The causal structure can be represented by the following CP-theory $T$:
\begin{align} 
(Prof: 0.7) &\leftarrow.\label{professor}\\
(Assistant: 0.8) &\leftarrow.\label{assistant}\\
NoPens &\leftarrow Prof \land Assistant.\label{nopens}
\end{align}
The first two of these causal laws are {\em vacuous} (i.e, they will be applied in every story) and {\em non-deterministic}. The first one results in either the professor taking the pen, or nothing at all. Since it is stated that professors typically do take pens, we simply chose some probability above $0.5$. Similarly, the second law describes the possible behaviour of the assistant. The last law is {\em deterministic}, i.e., it only has one possible outcome (where we leave implicit the probability $1$).

The given theory summarizes all possible {\em stories} that can take place in this model. One of those is what in fact did happen that Monday morning: both the professor and the assistant take a pen, leaving the receptionist faced with no pens. The other stories consist in only the professor taking a pen, only the assistant doing so, or neither.

To formalize this idea, the semantics of CP-logic uses {\em probability trees} \cite{shafer:book}. For this example, one such tree is shown in Figure \ref{fig:tree}. Here, each node represents a state of the domain, which is characterized by an assignment of truth values to the atomic formulas, in this case $Prof$, $Assistant$ and $NoPens$. In the initial state of the domain (the root node), all atoms are assigned their {\em default} value $\false$. In this example, initially there are still pens left, and neither the professor nor the assistant have taken one. The children of a node $x$ are the result of the application of a law: each edge $(x,y)$ corresponds to a specific disjunct that was chosen from the head of the law that was applied in node $x$. Because in this particular tree the professor arrives first, there are two child-nodes, corresponding to the law \eqref{professor} being applied and resulting in a pen being taken (left child) or not (right child). In the former case, the assignment is updated by setting $Prof$ to $\true$, its {\em deviant} value. Similarly, the subsequent nodes represent the possible outcomes of the application of law \eqref{assistant}, and whether or not this results in the lack of pens. The leftmost branch is thus the formal counterpart of the above story.
\begin{figure}
\centerline{
\xymatrix@H=0.2cm@C=3em{
&&\bullet \ar[ld]^{0.7}|(.3){\text{Prof takes\hspace*{0.7cm}}} \ar[rd]_{0.3}|(.3){\text{\hspace*{0.6cm}does not}}\\
&\bullet \ar[ld]_{0.8}|(.2){\text{Assistant takes\hspace*{0.95cm}}} \ar[d]^{0.2}|(.3){\hspace*{0.95cm}\text{does not}} && \bullet \ar[d]_{0.8}|(.2){\text{Assistant takes}\hspace*{0.95cm}} \ar[rd]^{0.2}|(.2){\hspace*{0.95cm}\text{does not}} \\
\bullet \ar[d]|(.3){\text{NoPens\hspace*{0.7cm}}} & \bullet && \bullet & \bullet \\
\circ  \\
}}
\caption{Probability tree for the Pen-vignette.\label{fig:tree}}
\label{fig:tree}
\end{figure}
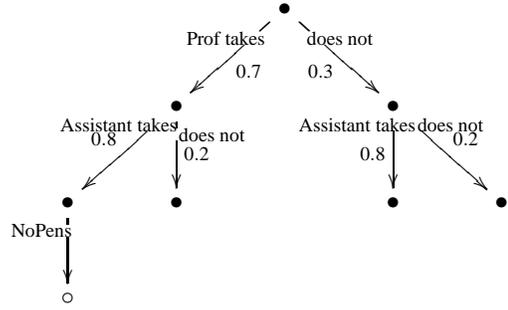
A probability tree of a theory $T$ defines an {\em a priori} probability distribution $P_T$ over all stories that might happen in this domain, which can be read off the leaf nodes of the branches by multiplying the probabilities on the edges. For instance, the probability of our example story is $0.56$. We have shown here only one such probability tree, but we can construct another one as well by applying the laws in different orders. 

An important property however is that all trees defined by the same theory result in the same probability distribution. To ensure that this property holds even when there are bodies containing negative literals, CP-logic makes use of the well-founded semantics. Simply put, this means the condition for a law to be applied in a node is not merely that its body is currently satisfied, but that this will remain so. This implies that a negated atom in a body should not only be currently assigned $\false$, but actually has to have become impossible, so that it will remain $\false$ through to the end-state. For atoms currently assigned $\true$, it always holds that they remain $\true$, hence here there is no problem.

\subsection{Operations on CP-theories}

We specify some operations on CP-logic theories that will be used throughout this paper.  Assume we have a theory $T$ and a branch $b$ of one of $T$'s probability trees, such that both $C$ and $E$ hold in its leaf. To make $T$ deterministic {\em in accordance with the choices made in $b$}, means to transform $T$ into $T^b$ by replacing the heads of the laws that were applied in $b$ with the disjuncts which were chosen from those heads in $b$. 

For our example story $b$, $T^b$ is:
\begin{align} 
Prof &\leftarrow.\\
Assistant &\leftarrow.\label{assdet}\\
NoPens &\leftarrow Prof \land Assistant.\label{nopens2}
\end{align}
We will use Pearl's $do()$-operator to indicate an intervention \cite{pearl:book}. The intervention on a theory $T$ that makes variable $C$ false, denoted by $do(\lnot C)$, removes $C$ from the head of any law in which it occurs, yielding $T|do(\lnot C)$.  For example, to prevent the professor from taking a pen, the resulting theory $T|do(\lnot Prof)$ is given by:
\begin{align} 
&\leftarrow.\label{empty}\\
(Assistant: 0.8)&\leftarrow.\\
NoPens &\leftarrow Prof \land Assistant.
\end{align}
Laws with an empty head, such as \eqref{empty}, can also simply be omitted. The analogous operation $do(C)$ on a theory $T$ corresponds to adding the deterministic law $C \leftarrow$.

\subsection{Actual Causation in CP-logic}\label{subsec:ac}

In recent years many proposals for defining actual causation have entered the scene, such as \cite{hall03,hall04,hall07,halpernpearl05a,vennekens11,beckers}. In \cite{hallpaper} we developed a general definition, which encompasses several of these. The idea behind our general definition is that one should first modify the theory $T$ into a theory $T^*$ that takes into account the influence of the actual story $b$, and then judge causation by using the probabilistic extension of counterfactual dependence. The available operations to construct $T^*$ are to make laws deterministic in accordance with the choices made in $b$, and to remove laws. A specific definition is obtained by specifying which laws fall under the first category -- called the {\em intrinsic} laws and noted as $Int$ -- and those which fall under the second -- called the {\em irrelevant} laws and noted as $Irr$.\footnote{These specifications are implicitly parametrized with regard to a theory, a story, and events $C$ and $E$.} We then construct $T^*$ as $T \setminus (Irr \cup Int) \cup Int^b$. 

\begin{definition}\label{acgen}[Actual causation] Given a theory $T$ and a branch $b$ such that both $C$ and $E$ hold in its leaf. We define that $C$ is an {\em actual cause} of $E$ in $(T,b)$ if and only if $P_{T^*}(\lnot E | do(\lnot C)) > 0$.\label{def:actc}
\end{definition}
If desired, the degree to which different events actually caused some effect can then be further compared by comparing their associated probabilities. 

HH use the HP-definition \cite{halpernpearl05a} as their working definition to illustrate their extension to actual causation. However, they stress the generality of their approach, and mention that one could for example apply it to Hall's definition from \cite{hall07}. In \cite{hallpaper}, we reformulated Hall's definition as an instantation of the general definition here presented. We will choose this reformulation to work with in this paper, because this makes it easier to apply the current discussion to other instantiations.

By $Leaf_b$ we denote the set of true variables in the leaf of a branch $b$. For a law $r$ that was applied in $b$, we call the node in $b$ that results from its application $Node_r^b$. Our working definition states that no laws are irrelevant, and determines the intrinsic laws as follows:

\begin{intrinsic}
A non-deterministic law $r$ of $T$ that was applied in $b$ is intrinsic iff there is no branch $d$ passing through a sibling of $Node_r^b$ such that $\{ C,E \} \subseteq Leaf_d \subseteq Leaf_b$.
\end{intrinsic}

If we look at the probability tree in Figure \ref{fig:tree}, we see that $E = NoPens$ only occurs in $Leaf_b$, and hence both the laws \eqref{professor} and \eqref{assistant} are intrinsic. This gives $T^*=T^b$. As there is counterfactual dependency of $NoPens$ on both $Assistant$ and $Prof$, they are full causes. However it was shown in \cite{knobe08} that most people do not judge the assistant to be a cause, or at least far less so than the professor is. This illustrates how our causal judgment can be influenced by normative considerations.

\section{The HH Extension to Actual Causation}\label{sec:hhac}

In this section we succinctly present the graded, context-dependent approach to actual causation from \cite{hh14}. As with much work on actual causation, HH frame their ideas using structural models. Such a model consists of a set of equations, one for each endogenous variable, that express the functional dependencies of the endogenous variables on others. These dependencies are acyclic, and have their roots in the exogenous variables. HH only consider structural models with two types of (Boolean) endogenous variables: the ones that deterministically depend on other endogenous variables, and those that depend directly on exogenous variables. 

HH take a {\em world} to be an assignment to all endogenous variables. Given a structural model, each assignment $u$ of the exogenous variables -- a so-called {\em context} -- determines a unique world, denoted by $s_u$. 

An extended structural model $(M,\succeq)$ consists of a structural model $M$ together with a normality ranking $\succeq$ over worlds. This ranking is a partial pre-order informed by our -- possibly subjective -- judgments about what we take to be normal in this context. It is derived by considering the typicality of the values the variables in a world take. For the variables that depend only on other endogenous variables, things are straightforward: it is typical that these variables take the value dictated by their deterministic equation. For the other variables, i.e., those that depend only on the exogenous variables, their possible values are ranked according to typicality. A world $s$ is more normal than $s'$ if there is at least one variable that takes a more typical value in $s$ than it does in $s'$, and no variable takes a less typical value. Typicality and normality are meant to encompass both statistical and normative judgments. Although HH make no syntactic distinction between the two kinds of normality, in the examples discussed they do differentiate between them informally. We will also make a formal distinction between the two, because in this manner we can incorporate information regarding both. 

Say we have a story, i.e., an assignment to all variables, such that $C$ and $E$ happen in it. In order to establish whether $C$ is a cause of $E$, any definition in the counterfactual tradition restricts itself to some particular set of counterfactual worlds in which $\lnot C$ holds and checks whether also $\lnot E$ holds in these worlds. If this set contains a world which serves to justify that $C$ is indeed a cause of $E$, then HH call such a world a {\em witness} of this. HH adapt a given definition of actual causation using the normality ranking to disallow worlds that are less normal than the actual world, in order to reflect the influence of normality on possible causes. 

 \begin{definition}\label{hh}[HH-extension of actual causation.] Given are an extended structural model $(M, \succeq)$ and context $u$, such that both $C$ and $E$ hold in $s_u$. $C$ is an HH-actual cause of $E$ in $(M, \succeq, u)$ iff $C$ is an actual cause of $E$ in $(M, u)$ when we consider only witnesses $w$ such that $w \succeq s_u$.
\end{definition}

Since we have a ranking on the normality of worlds, this definition straightforwardly leads to an ordering between different causes indicating the strength of the causal relationship by looking at the highest ranked witness for a cause, which is called its {\em best witness}. 

In case of our example, we get that $\{ \lnot Prof, Assistant, \lnot NoPens \} \succ \{ Prof, Assistant, NoPens \}$ $\succ \{ Prof, \lnot Assistant, \lnot NoPens \}$. Since the actual world is in the middle, the first of these can serve as a witness for $Prof$ being a cause, but the last may not be used to judge $Assistant$ to be a cause.  

\section{The HH Extension in CP-logic}\label{sec:hhcp}

We proceed with translating the HH-extension of actual causation into CP-logic. To get there, we will translate one by one all of the required concepts. 

A structural model $M$ in the HH-setting corresponds to a CP-theory $T$: a direct dependency on the exogenous variables results in a non-deterministic, vacuous law (such as \eqref{professor} and \eqref{assistant}), while a dependency on only endogenous variables results in a deterministic law (such as \eqref{nopens}).

A world $s$ described by $M$ then corresponds to a branch $b$ -- or to be more precise, the leaf of a branch -- of a probability tree of $T$. In Def. \ref{hh} we restrict an extended structural model to those worlds that are at least as normal as the actual world. The CP-logic-equivalent of this will be the {\em normal refinement} of $T$ according to $b$, which is a theory that describes those stories which are at least as normal as $b$. 

First we introduce two operations on CP-laws, corresponding to the two different interpretations of normality. Assume at some point a CP-law $r$ was applied in $b$, choosing the disjunct $r_b$ that occurs in its head. 

On a probabilistic reading, an alternative application of $r$ is at least as normal as the actual one if a disjunct is chosen which is at least as likely as $r_b$. Therefore the {\em probabilistically normalized} refinement of $r$ according to $b$ -- denoted by $r^{PN(b)}$ -- consists in $r$ without all disjuncts that have a strictly smaller probability than $r_b$. Remain the laws that were not applied in $b$. In line with the understanding of normality from HH, we choose to handle these such that they cannot have effects which result in a world less normal than the actual world. Thus we remove those disjuncts which have a probability lower than $0.5$, and are inconsistent with the assignment in the leaf of $b$. Further, say the total probability of the removed disjuncts in some law is $p$, then we renormalize the remaining probabilities by dividing by $1-p$ (Unless $p=1$, then we simply remove the law.) 

\begin{definition}
Given a theory $T$, and a story $b$, we define the {\em probabilistically normalized refinement} of $T$ according to $b$ as $T^{PN(b)}:=\{ r^{PN(b)} | r \in T\}$.
\end{definition}
In case of our example, $T^{PN(b)} = T^b$, shown earlier.

A second reading of normality considers not what did or could happen, but what ought to happen. To allow such considerations, we extend CP-logic theories with {\em norms}. Everything that can possibly happen is described by the CP-laws of a theory, hence we choose to introduce {\em prescriptive} norms as corrections to {\em descriptive} CP-laws. These corrections take the form of alternative probabilities for the disjuncts in the head of a law, which represent how the law should behave. (A more general approach could be imagined, but for the present purpose this extension will suffice.) These probabilities will be enclosed in curly braces, and have no influence on the actual behaviour of a theory. However, if we wish to look at how the world should behave, we can enforce the norms by replacing the original probabilities of a theory $T$ with the normative ones. For example, extending law \eqref{professor} with the norm that the professor shouldn't take pens, even though he often does, gives $(Prof: 0.7 \{p\}) \leftarrow$, where $1 \gg p$. The {\em normatively normalized} refinement is then given by the CP-law $(Prof: p)\leftarrow$. To properly capture the HH definition, our normalized theory should allow all worlds that are at least as normal as $b$, including of course $b$ itself. For this reason, we here restrict attention to norms with $0 < p < 1$, because a norm $p=0$ or $p=1$ could make the actual world $b$ impossible. (This restriction is lifted in our own proposal in Section \ref{sec:final}).

\begin{definition}
Given an extended theory $T$, i.e., a theory also containing norms, we define the {\em normatively normalized refinement} of $T$ as $T^{NN}:=\{ r^{NN} | r \in T\}$, where $r^{NN}$ is the normatively normalized refinement of $r$.
\end{definition}

We can combine both senses of normality, as follows:

\begin{definition}
Given an extended theory $T$, and a story $b$, we define the {\em normal refinement} of $T$ according to $b$ as $T^{Normal(b)}:=(T^{NN})^{PN(b)}$.
\end{definition}
The normal refinement according to $b$ is constructed out of $T$ by eliminating all the disjuncts with values of variables that are less normal (either probabilistically or normatively) than the values from $b$, and thus it allows precisely those stories which are at least as normal as $b$.\footnote{A formal proof of this can be found in the Appendices, as Lemma \ref{lem2}.} In case of our example, $T^{Normal(b)}$ is given by $(Prof: p)\leftarrow$ and laws \eqref{assdet} and \eqref{nopens2}.

Remains to be explained how we get from a given extended structural model to an extended CP-theory. The normality ranking is derived by considering what is typical for those variables depending directly on the exogenous variables, i.e., those that we represent by $X :p \leftarrow$. HH use statements that take the form: ``it is typical for the variable $X$ to be $\true$", or ``it is typical for it to be $\false$". In CP-logic this becomes: $p > 0.5$, and $p < 0.5$ respectively. A statement of the form: ``it is more typical for $X$ to be $\true$ than for $Y$ to be $\true$" translates in an ordering on the respective probabilities. If the typicality statement is of the normative kind, then it is best represented by norms in CP-logic. Thus if there is a norm regarding $X$ then the law will take the extended form $X :p \{q \} \leftarrow$.

In Def. \ref{hh}, we have that a world is an acceptable witness only if it belongs to the set of worlds allowed by the definition of actual causation, and it is at least as normal as the actual world. Similarly, we need to limit the stories allowed by Def. \ref{acgen}  -- which are described by $T^*$ -- to those stories which are at least as normal as $b$. We would like to do this in exactly the same manner as we did for $T$, i.e., by looking at $(T^{*})^{Normal(b)}$. 
Unfortunately we need to treat the unique law that contains $C$ in its head -- denoted by $r(C)$ -- somewhat special. That's because in Def. \ref{acgen} the intervention $do(\lnot C)$ makes this law deterministic with an empty head. As we want to take into account the (ab)normality of $\lnot C$, we work with the normal refinement instead. Hence we denote by $T^{**}$ the theory identical to $T^*$ except that it contains the normal refinement of $r(C)$.

\begin{definition} Given an extended theory $T$, a story $b$ such that $C$ and $E$ hold in its leaf, and the theory $T^*$ as described in Section \ref{subsec:ac}. We define the {\em normal refinement} of $T^*$ according to $b$ and $C$ as $T^{Normal(b)*} := (T^{**})^{Normal(b)}$.
\end{definition}

Besides the normality and belonging to $T^*$, recall that the remaining requirement for a story to be a witness, is that $\lnot C$ and $\lnot E$ hold in it. This leads us to the following formulation of the HH-approach in CP-logic.

\begin{definition}\label{hhcplogic}[HH-CP-logic-extension of actual causation] Given an extended theory $T$, and a branch $b$ such that both $C$ and $E$ hold in its leaf. We define that $C$ is an HH-CP-logic-actual cause of $E$ in $(T,b)$ iff $P_{T^{Normal(b)*}}(\lnot E \land \lnot C) > 0$.
\end{definition}

For $C = Assistant$, the probabilistic normalisation of $T^{**}$ replaces $r^{C}$ with the deterministic law $Assistant \leftarrow$, leading to $P(\lnot NoPens \land \lnot Assistant) = 0$. On the other hand, for $C = Professor$, the normative normalisation replaces $r^{C}$ with $(Prof: p) \leftarrow$, leading to $P(\lnot NoPens \land \lnot Prof) = q$. Thus $Prof$ is judged to be a strong cause of $NoPens$, whereas $Assistant$ isn't a cause at all, in line with the empirical results from \cite{knobe08}. Note that it is only by using the normative probabilities rather than the statistical ones that we get the correct response for $Prof$.

We now show that Def. \ref{hhcplogic} is indeed the correct translation of the HH approach from structural models to CP-logic.

\begin{theorem} $C$ is an HH-actual cause of $E$ in an extended model and context $(M, \succeq, u)$ iff $C$ is an HH-CP-logic-actual cause of $E$ in $(T,b)$, where $(T, b)$ is derived from $(M, \succeq, u)$ in the sense described above.
\end{theorem}

\begin{proof}
See Appendix.
\end{proof}

\section{The Importance of Counterfactuals}\label{sec:cd}

We mentioned earlier that one criterion for a story to be normal was that it respects the laws/equations. On the other hand definitions of actual causation look at counterfactual stories resulting from an intervention, namely $do(\lnot C)$, which violates the laws. Following HH, Def. \ref{hhcplogic} tries to circumvent the use of this intervention by simply demanding that $\lnot C$ holds. However this solution is not always available, and when it isn't this provides counterintuitive results. We illustrate what goes wrong by using the following theory:

\begin{minipage}{0.18\textwidth}
\begin{align*} 
(A: 0.1) &\leftarrow.\\
C &\leftarrow A.
\end{align*}
\end{minipage}
\begin{minipage}{0.18\textwidth}
\begin{align*} 
E &\leftarrow C.\\
E &\leftarrow \lnot A.
\end{align*}
\end{minipage}
\\ \\
Consider the story where first $A$ occurs, followed by $C$ and $E$. Intuitively, $C$ is a strong cause of $E$, also when taking into account the typicality of $C$. The law with $A$ in its head is intrinsic, and thus $T^{Normal(b)*}$ is:

\begin{minipage}{0.18\textwidth}
\begin{align*} 
A &\leftarrow.\\
C &\leftarrow A.
\end{align*}
\end{minipage}
\begin{minipage}{0.18\textwidth}
\begin{align*} 
E &\leftarrow C.\\
E &\leftarrow \lnot A.
\end{align*}
\end{minipage}
\\ \\
Applying the definition, we get that $P_{T^{Normal(b)*}}(\lnot E \land \lnot C) = 0$, giving the absurd result that $C$ is not a cause of $E$ at all. The problem lies in the fact that in its current form we only allow stories containing $\lnot C$ in the usual, lawful way, rather than stories which contain $\lnot C$ as a result of the intervention $do(\lnot C)$. The problem remains if we use the HP-definition -- as HH does -- instead of our working definition.

We can set this straight by looking instead at $P_{T^{Normal(b)*}}(\lnot E | do(\lnot C))$, so that we re-establish the counterfactual nature of our definition. (As $T^{**}|do(\lnot C)=T^{*}|do(\lnot C)$, this is equivalent to $P_{(T^*)^{Normal(b)}}(\lnot E | do(\lnot C))$, which no longer mentions the artificial theory $T^{**}$.) However, by making this move we no longer take into account the (ab)normality of $C$ itself, whereas research shows extensively that causal judgments regarding an event are often influenced by how normal it was \cite{kahn86,knobe08,hitchcockknobe09}. (This effect is not limited to normative contexts. For example, the lighting of a match is usually judged a cause of a fire, whereas the presence of oxygen is considered so normal that it isn't.) Hence we should factor in this normality, which is expressed by $P_{T^{Normal(b)}}(\lnot C))$. As the following theorem shows, our new choice only makes a difference in a limited set of cases.

\begin{theorem} If $r(C)$ is non-deterministic or $P_{T^{Normal(b)}}(\lnot C) = 0$, then $P_{T^{Normal(b)*}}(\lnot E \land \lnot C) = P_{(T^*)^{Normal(b)}}(\lnot E | do(\lnot C))*P_{T^{Normal(b)}}(\lnot C)$.
\end{theorem}

\begin{proof}
See Appendix.
\end{proof}

If $r(C)$ is deterministic and $P_{T^{Normal(b)}}(\lnot C) > 0$, as in the example shown, then contrary to the left-hand side of the equation, the proposed adjustment on the right-hand side of the equation gives the desired result $1*0.9=0.9$.

\section{The Importance of Probabilities}\label{sec:norms}

Because the HH-approach lacks the quantification of normality offered by probabilities, they dismiss entirely all witnesses that are less normal than the actual world. A direct consequence is that any typical event -- i.e., $P > 0.5$ -- is never a cause, which is quite radical. By using probabilities, this qualitative criterion is no longer necessary: less normal witnesses simply influence our causal judgment less. Further, HH order causes solely by looking at the best witnesses. We now present an example which illustrates the benefit of both abandoning their criterion, and aggregating the normality of witnesses to order causes, without sacrificing the influence of normality.

\subsection{Why it is better to use $(T^{*})^{NN}$ in the first factor}

\begin{quote}
Imagine you enter a contest. If a $10$-sided die lands $1$, you win a car. If not, you get a $100$ more throws. If all of them land higher than $1$, then you also win the car. The first throw lands $1$, and you win the car.
\end{quote}

It's hard to imagine anyone objecting to the judgment that the first throw is a cause of you winning the car. Yet that is exactly what we get when applying the current definition. The following theory $T$ describes the set-up of the contest, where $Throw(i,j)$ means that the $i$-th throw landed $j$ or smaller.
\begin{align*} 
(Throw(1,1): 0.1) &\leftarrow.\\
(Throw(2,1): 0.1) &\leftarrow \lnot Throw(1,1).\\
(Throw(3,1): 0.1) &\leftarrow \lnot Throw(1,1) \land \lnot Throw(2,1).\\
...&\\
WinCar &\leftarrow Throw(1,1).\\
WinCar &\leftarrow \lnot Throw(2,1) \land ...\land \lnot Throw(100,1).
\end{align*}

The normal refinement of $T$ according to the story is given by:
\begin{align*} 
(Throw(1,1): 0.1) &\leftarrow.\\
WinCar &\leftarrow Throw(1,1).\\
WinCar &\leftarrow \lnot Throw(2,1) \land  ...\land \lnot Throw(100,1).
\end{align*}

We get that $P_{(T^*)^{Normal(b)}}(\lnot WinCar | do(\lnot Throws(1,1)))= 0$, and thus $Throws(1,1)$ is not a cause of $WinCar$. In terms of HH: although $\lnot Throw(1,1) \land \lnot Throw(2,1) \land  ...\land \lnot Throw(100,1) \land WinCar$, is very unlikely, it is the only candidate witness. To see why, recall that a witness needs to have $\lnot Throws(1,1)$, and should be at least as normal as the actual world. In every other world with $\lnot Throws(1,1)$, at least one of the $Throws(i,1)$ is true, and hence it is less normal. But in a witness it should hold that $\lnot WinCar$, so there is no witness for $Throws(1,1)$ being a cause of $WinCar$.

On the other hand the theory $(T^*)^{NN}$ in this case is simply equal to $T$, but for the first law being $Throw(1,1) \leftarrow$. Hence the probability of not winning the car given that the first throw does not land $1$ is pretty much $1$, and the value in the equation becomes approximately $0.9$, indicating $Throw(1,1)$ to be a very strong cause of $WinCar$. 

\subsection{Why we should look at all witnesses}

The current example also illustrates why it makes sense to aggregate the strength of all witnesses: the best witness for $Throw(1,1)$ is the story in which $\lnot Throw(1,1) \land Throw(2,1) \land \lnot Throw(3,1) \land ...\land \lnot Throw(100,1)$, with a probability of $0.09$, thus making $Throw(1,1)$ a very minor cause of $WinCar$ under the HH-approach. Put informally, the alternative we should be considering is not one particular outcome -- the best witness -- of the throwing sequence that makes you lose the car, but the set of all such sequences -- all witnesses -- taken together.

\subsection{Why it is better to use $T^{NN}$ in the second factor}

Imagine the same story, with a slight variation to the rules of the contest: you win the car on the first throw if the die lands anything under $7$. Hence the first head changes to $Throw(1,6) : 0.6$, making it a typical outcome. Therefore the first law becomes deterministic in $T^{PN(b)}$, giving that $P_{T^{Normal(b)}}(\lnot Throw(1,6)) = 0$, which again results in the counterintuitive judgment that the first throw in no way caused you to win the car. 

We therefore suggest to use $T^{NN}$ in the second factor of the inequality rather than $T^{Normal(b)}$, making use of the gradual measurement offered by probabilities. Applying this idea to the example, we get the result that $Throw(1,6)$ has causal strength $0.4$. This value is smaller than before, because the cause is now less atypical. 

\subsection{The final definition}\label{sec:final}

This brings us to our final extension to a definition of actual causation. 

\begin{definition}[Extension of actual causation] Given an extended theory $T$, and a branch $b$ such that both $C$ and $E$ hold in its leaf. We define that $C$ is an actual cause of $E$ in $(T,b)$ if and only if $P_{(T^*)^{NN}}(\lnot E | do(\lnot C))*P_{T^{NN}}(\lnot C) > 0$.
\end{definition}

\section{Conclusion}

Our final definition extends our original definition of actual causation (Def. \ref{acgen}) by incorporating the main points raised by HH: (1) it allows normative considerations and (2) is able to factor in the normality of the cause, which is important when we are considering interventions and explanatory power. In addition, it also improves on the HH account in several ways:

\begin{itemize}
\item{CP-logic is more expressive than structural models, hence it can be applied to more examples \cite{vennekens:jelia}.}
\item{Separating normative from statistical normality allows for a more accurate description of the domain.}
\item{Since we no longer refer to the actual world in the second factor, we can use strict norms.}
\item{It is able to deal with all of the examples from the HH-paper equally well as Def. \ref{hhcplogic}.}
\item{It can also handle the previous examples properly, as opposed to Def. \ref{hhcplogic}.}
\end{itemize}

\section{Acknowledgements}
Sander Beckers was funded by a Ph.D. grant of the Agency for Innovation by Science and Technology (IWT-Vlaanderen). 

\appendix
\section{Appendices}

\setcounter{theorem}{1}
\addtocounter{theorem}{-1}

Assume that, for a structural model $M$ and context $u$, we have defined in some way the set of counterfactual worlds $W_{M,u}(C,E)$ that are relevant to decide whether $C$ actually causes $E$ in $(M,u)$. We can then define that $C$ actually causes $E$ in $(M,u)$ if and only if there exists some witness $w \in W_{M,u}(C,E)$ for which $w \models \lnot C \land \lnot E$. Assume also that we have a corresponding definition in the context of CP-logic: for a CP-logic theory $T$ and branch $b$, we have defined the set of counterfactual branches $B_{T,b}(C,E)$, and say that $C$ actually caused $E$ in $b$ if and only if there exists a witness $b' \in B_{T,b}(C,E)$ such that $Leaf_b' \models \lnot C \land \lnot E$. Moreover, assume that these two notions are equivalent, i.e., that $w \in W_{M,u}(C,E)$ if and only if there exists a branch $b \in B_{T,b}(C,E)$ such that $Leaf_b=w$.

To facilitate the proof of Theorem \ref{theorem1}, we introduce the following lemma.

\begin{lemma}\label{lem2} Given an extended model and context $(M, \succeq, u)$, and a theory and branch $(T, b)$ that are derived from $(M, \succeq, u)$ in the sense described earlier. Then for any world $w$, and a branch $d$ of a probability tree from $T$ that corresponds to it, it holds that $w \succeq s_u$ iff $d$ occurs in a probability tree of $T^{Normal(b)}$.
\end{lemma}

\begin{proof}
We know that $b$ is a branch in a probability tree from $T$ such that $Leaf_b$ has the same assignment as $s_u$. Recall that $T$ consists of two categories of laws. First there are those corresponding to the equations for the endogenous variables which depend on other endogenous variables, which are deterministic and thus re-appear in $T^{Normal(b)}$ unchanged. Second there are those corresponding to the endogenous variables which directly depend on the exogenous variables, which take the form $X: p\{q\} \leftarrow$, where the second probability need not be present. 

Assume we have a world $w$ such that $w \succeq s_u$. Any world that satisfies the equations of $M$ follows deterministically from a context, i.e., an assignment to all exogenous variables. As $s_u$ is a world that satisfies the equations, and $w$ is at least as normal, it also satisfies the equations. Hence there is a context $u'$ which determines $w$. In CP-logic, such a context corresponds to choosing particular disjuncts in the heads of all laws from the second category. 

Concretely, this means that for each law/equation of the second category, the value of the corresponding variable $X$ is at least as typical in $w = s_u'$ as it is in $s_u$. Denote by $X_w$ and $X_s$ the values $X$ takes in the worlds $w$ and $s_u$ respectively. By construction of $T^{Normal(b)}$, the disjuncts which are at least typical as $X_s$ -- be it in the statistical or in the normative sense -- still appear in the law for $X$ in $T^{Normal(b)}$, and hence can be chosen when this law is applied. Therefore the branches corresponding to $w$ from the probability trees of $T$ also appear in the probability trees of $T^{Normal(b)}$, be it that the values of the probabilities may have changed. 

Now assume we have a branch $d$ corresponding to a world $w$, that occurs in a probability tree of $T^{Normal(b)}$. We can simply reverse the correspondence between the choices of disjuncts and a context, to obtain that $w \succeq s_u$.  

\end{proof}

\begin{theorem}\label{theorem1} $C$ is an HH-actual cause of $E$ in an extended model and context $(M, \succeq, u)$ iff $C$ is an HH-CP-logic-actual cause of $E$ in $(T,b)$, where $(T, b)$ is derived from $(M, \succeq, u)$ in the sense described in Section 4.
\end{theorem}

\begin{proof}
We begin with the implication from left to right. So assume we have an extended model and context $(M, \succeq, u)$, such that  $C$ and $E$ hold in $s_u$, and there is at least one witness $w$ of $C$ being an actual cause of $E$ in $(M, u)$ such that $w \succeq s_u$. 

By the assumptions made above, we get that $C$ is an actual cause of $E$ in $(T, b)$, and more specifically that any branch $d$ that corresponds to $w$ is a witness of this. Thus $d$ appears in a probability tree of $T^*|do(\lnot C)$. 

By Lemma \ref{lem2}, we know that such a branch $d$ also appears in a probability tree of $T^{Normal(b)}$. 

We look separately at the two options regarding $r(C)$. First we assume that $r(C)$ is non-deterministic. Since $d$ occurs in a tree of $T^{Normal(b)}$, and $\lnot C$ holds in it, the empty disjunct remains present in the normal refinement of $r(C)$. By definition, $T^{**}$ is simply $T^*$ with the normal refinement of $r(C)$. Therefore $d$ also occurs in a tree of $T^{**}$. 

Second, assume $r(C)$ is deterministic. Then $T^{**}=T^*$. Since $d$ occurs in a tree of $T^{Normal(b)}$, which obviously contains $C$ in the head of $r(C)$, the body for $r(C)$ cannot be satisfied in $d$. Thus the intervention $do(\lnot C)$ is irrelevant to $d$, and again we can conclude that $d$ also occurs in a tree of $T^{**}$. 

So in all cases we have that $d$ occurs both in a tree of $T^{Normal(b)}$, and in a tree of $T^{**}$. This implies that the disjuncts chosen in the laws applied in $d$ occur in the versions these laws take in both of these theories, with possibly different but strictly positive probabilities. Note that every law from $T^{Normal(b)*}$ either takes the form it has in $T^{**}$ or it takes the form it has in $T^{Normal(b)}$. Therefore $d$ also appears in $T^{Normal(b)*}$. It being a witness, $\lnot C$ and $\lnot E$ hold in it, and thus the stated probability is strictly positive.

Now we continue with the reverse implication. Assume we have an extended theory $T$, a story $b$ such that $C$ and $E$ hold in it, and $P_{T^{Normal(b)*}}(\lnot E \land \lnot C) > 0$. This implies the existence of a branch $d$ in $T^{Normal(b)*}$ such that both $\lnot C$ and $\lnot E$ holds. 

Say $r$ is a law from $T^{Normal(b)*}$. If $r$ is intrinsic and not $r(C)$, it is deterministic, containing the single (possibly empty) disjunct $r_d$ with associated probability $1$. As $r_d$ was the actual choice from $b$, by construction $r_d$ also appears in the normal refinement of $r$, although the probability may be different. However, as long as we do not have strict norms, i.e., norms where $p$ or $q$ is $1$, this probability will be strictly positive. A strict norm means that a violation of it is considered entirely abnormal, analogous to the occurrence of an event with zero probability. Since HH treat norms identical to statistical normality, and since the actual world was possible, it follows that the actual world is not entirely abnormal. Hence even if $r_d$ was a violation of a norm, it will not have been a strict norm. (Our final definition from Section 5.5 does allow for strict norms.) Thus, we conclude that $r_d$ occurs in the head of the versions of the law $r$ we find in both $T^*$ and $T^{Normal(b)}$. Because $r$ is not $r(C)$, we can say the same about $T^*|do(\lnot C)$.

If $r$ is not intrinsic and not $r(C)$, it contains all of its original disjuncts when it occurs in $T^*$. Therefore it takes the same form in $T^{Normal(b)*}$ as it does in $T^{Normal(b)}$. Again we conclude that $r_d$ occurs in the head of the versions of the law $r$ we find in each of $T^*$, $T^{Normal(b)}$ and $T^*|do(\lnot C)$.

This leaves us to consider $r(C)$. By definition, $T^{Normal(b)*}$ contains the same version as $T^{Normal(b)}$. From this and the previous paragraphs we can already conclude that any branch occuring in a tree of $T^{Normal(b)*}$ also occurs in a tree of $T^{Normal(b)}$. More specifically this holds for $d$. Thus by Lemma \ref{lem2}, it holds for the corresponding world $w$ that $w \succeq s_u$.

If the body for $r(C)$ is false in $d$, then the precise form of the head of $r(C)$ is irrelevant for $d$. As the head of $r(C)$ is the only difference between $T^{**}$ and  $T^*|do(\lnot C)$, we can again conclude that $d$ also occurs in $T^*|do(\lnot C)$. 

Leaves us to consider the case that the body for $r(C)$ is true in $d$. From the fact that $d$ -- in which $\lnot C$ holds -- occurs in $T^{Normal(b)}$, we can infer that $r(C)$ is a non-deterministic law. Taken together with the knowledge that the disjunct containing $C$ was chosen in $b$, it follows that the normal refinement of $r(C)$ contains both $C$ and the empty disjunct in its head. Furthermore, in $d$ the empty disjunct was chosen. These observations taken together imply that the disjunct of $r(C)$ chosen in $d$ occurs in the head of the versions of $r(C)$ we find in both $T^{**}$ and $T^*|do(\lnot C)$. Once more we conclude that $d$ also occurs in $T^*|do(\lnot C)$.

Thus $d$ is a witness for $C$ being an actual cause of $E$ in $(T,b)$. Therefore the world $w$ corresponding to $d$ is a witness for $C$ being an actual cause of $E$ in $(M,u)$. Together with the fact that $w \succeq s_u$, the conclusion follows. 

\end{proof}

\begin{theorem} If $r(C)$ is non-deterministic or $P_{T^{Normal(b)}}(\lnot C) = 0$, then $P_{T^{Normal(b)*}}(\lnot E \land \lnot C) = P_{(T^*)^{Normal(b)}}(\lnot E | do(\lnot C))*P_{T^{Normal(b)}}(\lnot C)$.
\end{theorem}

\begin{proof}

First we examine the case where $P_{T^{Normal(b)}}(\lnot C) = 0$. This implies that the right-hand side of the equation is $0$. Also, any branch from a tree $T^{Normal(b)*}$ occurs as well in a tree of $T^{Normal(b)}$, so $P_{T^{Normal(b)*}}(\lnot C) = 0$ and the left-hand side is also equal to $0$. 

This leaves us to consider the case where $P_{T^{Normal(b)}}(\lnot C) > 0$ and  $r(C)$ is non-deterministic.

In this case $P_{T^{Normal(b)*}}(\lnot C) = P_{T^{Normal(b)}}(\lnot C)$, so we have: $P_{T^{Normal(b)*}}(\lnot E \land \lnot C) = P_{T^{Normal(b)*}}(\lnot E \land \lnot C)*P_{T^{Normal(b)}}(\lnot C)/P_{T^{Normal(b)*}}(\lnot C) = P_{T^{Normal(b)*}}(\lnot E | \lnot C)*P_{T^{Normal(b)}}(\lnot C)$. 

Further, conditioning on $\lnot C$ when $C$ only occurs in a vacuous non-deterministic law, is identical to looking at the intervention $do(\lnot C)$, thus the list of equalities continues: 

$ = P_{T^{Normal(b)*}}(\lnot E | do(\lnot C))*P_{T^{Normal(b)}}(\lnot C)$. Also, $T^{Normal(b)*}|do(\lnot C)=(T^*)^{Normal(b)*}|do(\lnot C)$, which brings us to the desired conclusion.

\end{proof}

\bibliographystyle{aaai}
\bibliography{aaai_mybib}

\end{document}